\pdfoutput=1
\documentclass{article}

\usepackage{microtype}
\usepackage{graphicx}
\usepackage{subfigure}
\usepackage{booktabs} 

\usepackage{hyperref}

\usepackage[preprint]{icml2026}

\usepackage{amsmath}
\usepackage{amssymb}
\usepackage{mathtools}
\usepackage{amsthm}
\usepackage{algorithm}
\usepackage{algorithmic}

\theoremstyle{plain}
\newtheorem{theorem}{Theorem}
\newtheorem{proposition}[theorem]{Proposition}
\theoremstyle{definition}

\DeclareMathOperator*{\argmin}{arg\,min}
\DeclareMathOperator{\Acc}{Acc}
\DeclareMathOperator{\Cost}{Cost}
\DeclareMathOperator{\ECE}{ECE}

\icmltitlerunning{UCCI: Calibrated Uncertainty for Cost-Optimal LLM Cascade Routing}

\begin{document}

\twocolumn[
\icmltitle{UCCI: Calibrated Uncertainty for\\Cost-Optimal LLM Cascade Routing}

\icmlsetsymbol{equal}{*}

\begin{icmlauthorlist}
\icmlauthor{Varun Kotte}{ind}
\end{icmlauthorlist}

\icmlaffiliation{ind}{Independent Researcher}

\icmlcorrespondingauthor{Varun Kotte}{kottevarun@gmail.com}

\icmlkeywords{LLM routing, model cascades, uncertainty calibration, cost-aware inference, isotonic regression, named entity recognition}

\vskip 0.3in
]

\printAffiliationsAndNotice{}

\begin{abstract}
LLM cascades and model routing promise lower inference cost by sending
easy queries to a small model and escalating hard ones to a large model,
but most deployed routers use uncalibrated confidence scores and require
per-workload threshold tuning. We present UCCI, a calibration-first router
that maps token-level margin uncertainty to a per-query error probability
via isotonic regression and selects the escalation threshold by constrained
cost minimization. Under three explicit assumptions, threshold policies on
the calibrated score are cost-optimal, and isotonic calibration achieves
$O(n^{-1/3})$ sample complexity for expected calibration error (ECE). On a
production named entity recognition workload of 75{,}000 queries served by
4B and 12B instruction-tuned LLMs on H100 GPUs, UCCI cuts inference cost by
31\% (95\% CI: [27\%, 35\%]) at micro-F1 = 0.91 while reducing ECE from
0.12 to 0.03. At the same operating point, UCCI beats entropy thresholding,
split-conformal routing, and a FrugalGPT-style learned threshold. All
cascade results use end-to-end routing on actual model outputs and measured
H100 latency, not simulated routing from global accuracies or nominal API
prices.
\end{abstract}

\section{Introduction}
\label{sec:intro}

Large language models perform well on structured prediction tasks but
their per-query inference cost grows steeply with model size. For
deployments where latency and serving cost matter, \emph{cascaded
inference}, serving most requests with a small model and escalating
hard cases to a large one, is an attractive option
\citep{chen2023frugalgpt}.

Cascade quality is bounded by the routing signal. LLM confidence
scores, including raw token entropy and token margins, are typically
miscalibrated and are sensitive to prompt wording
\citep{guo2017calibration,jiang2021trust}. As a consequence, threshold
choices that work in one workload often fail in another, and
practitioners retune by hand for each deployment. The result is a gap
between the theoretical promise of cascades and what production systems
actually achieve.

This paper closes that gap by treating cascade routing as a calibrated
decision problem rather than a tuning exercise. We make four claims,
each supported by results in Section~\ref{sec:experiments}:

\begin{itemize}
\item \textbf{Formulation.} We pose cascade design as expected cost
  minimization subject to an accuracy constraint, and prove that
  threshold policies on calibrated error probabilities are cost-optimal
  under three explicit assumptions (Theorem~\ref{thm:opt}).
\item \textbf{Calibration.} We map token-margin uncertainty to error
  probability via isotonic regression, reducing expected calibration
  error from 0.12 to 0.03 on our workload.
\item \textbf{End-to-end evaluation.} On a production named entity
  recognition (NER) workload with 75{,}000 labeled queries and 6 entity
  types, UCCI cuts cost by 31\% (95\% CI: [27\%, 35\%]) at micro-F1 =
  0.91 versus large-model-only inference. Costs come from measured
  latency on H100 GPUs, and routing uses actual model outputs at test
  time, not simulated outcomes from global accuracies.
\item \textbf{Comparison.} At the same F1 = 0.91 operating point, UCCI
  beats entropy thresholding (cost 2.31), conformal-prediction routing
  (2.18), and a FrugalGPT-style learned threshold (2.24).
\end{itemize}

Operationally, UCCI suggests calibrating first and thresholding second. A raw
token-margin signal is weak when used directly, but after a small isotonic
fit on a held-out calibration set it becomes a near-optimal routing score
in our setting. Most of the cascade engineering value comes from making the
routing score probabilistic, not from hand-tuning the final threshold.

\section{Related Work}
\label{sec:related}

\textbf{Selective prediction and learning to defer.} Selective
classification \citep{geifman2017selective} and learning-to-defer
\citep{madras2018predict,mozannar2020consistent} study prediction with
abstention or expert deferral under cost or risk constraints. These
frameworks assume scalar confidence scores on classification heads.
UCCI adapts the same idea to LLM structured outputs, where
(i)~uncertainty must be aggregated from token-level generation,
(ii)~raw token confidences are miscalibrated and require explicit
correction, and (iii)~production deployment makes end-to-end
measurement, rather than simulated routing on global accuracies, the
right validation target.

\textbf{Uncertainty quantification and calibration.} Calibration maps
confidence scores to probabilities of correctness
\citep{guo2017calibration,zadrozny2002transforming,naeini2015obtaining,platt1999probabilistic}.
For LLMs, token probabilities are often miscalibrated
\citep{jiang2021trust} and uncertainty estimation for free-form
generation remains active \citep{kuhn2023semantic,farquhar2024semantic}.
UCCI combines
token-margin aggregation with isotonic regression to produce calibrated
error probabilities (ECE = 0.03 on our workload) suitable for
cost-constrained routing. Reliability layers for LLM structured
extraction are studied in a complementary line of work
\citep{kotte2026promptport}.

\textbf{Conformal prediction for selective classification.} Recent work
applies conformal prediction to selective classification with coverage
guarantees \citep{angelopoulos2021gentle,sadinle2019least}. We compare
UCCI against a conformal-routing baseline that uses raw token-margin
uncertainty as the nonconformity score.

\textbf{LLM cascades and cost-aware routing.} FrugalGPT
\citep{chen2023frugalgpt} routes between commercial APIs using learned
confidence thresholds but does not impose calibration on the routing
score. Hybrid LLM \citep{ding2024hybrid} learns a quality-aware router
between a small and a large model, and RouteLLM \citep{ong2024routellm}
learns routers from preference data. RouterBench \citep{hu2024routerbench}
standardizes evaluation for heterogeneous multi-LLM routing systems.
Concurrent cascade work studies early abstention in settings where even
the large model may decline to answer \citep{zellinger2025early}. Cascade
ideas have also been applied to early exiting in language models
\citep{schuster2022confident} and to multi-model speculative decoding
\citep{chen2023cascade}. Cost-aware deployment is a recurring concern in
retrieval-augmented LLM systems for closed-domain applications
\citep{sharma2024rag,kotte2025patent}. UCCI differs from these systems by
making the routing score an explicitly calibrated error probability and by
proving when a threshold on that probability is cost-optimal.

\section{Problem Formulation}
\label{sec:problem}

We consider inputs $x \sim \mathcal{D}$ for a structured prediction
task with ground truth $y$. A small model $f_s$ and a large model
$f_\ell$ produce predictions $\hat{y}_s = f_s(x)$ and $\hat{y}_\ell =
f_\ell(x)$. Let $\Acc(\hat{y}, y) \in [0, 1]$ denote an accuracy
metric (we use micro-averaged F1). Let $c_s$ and $c_\ell$ denote
per-query inference costs with $c_\ell > c_s$.

A routing policy $\pi : \mathcal{X} \to \{s, \ell\}$ selects which
model to use:
\begin{equation}
\hat{y}_\pi(x) =
\begin{cases}
f_s(x) & \text{if } \pi(x) = s\\
f_\ell(x) & \text{if } \pi(x) = \ell
\end{cases}
\end{equation}
with cost $C_\pi(x) = c_s \, \mathbf{1}\{\pi(x) = s\} + c_\ell \,
\mathbf{1}\{\pi(x) = \ell\}$ and accuracy $\Acc_\pi(x) =
\Acc(\hat{y}_\pi(x), y)$.

\textbf{Cost-optimal cascade design.} Given a target accuracy $\tau
\in [\alpha_s, \alpha_\ell]$ where $\alpha_s = \mathbb{E}[\Acc(f_s(x),
y)]$ and $\alpha_\ell = \mathbb{E}[\Acc(f_\ell(x), y)]$, we seek
\begin{align}
\min_{\pi} \quad & \mathbb{E}[C_\pi(x)] \label{eq:obj}\\
\text{s.t.} \quad & \mathbb{E}[\Acc_\pi(x)] \geq \tau. \label{eq:constraint}
\end{align}

\section{Uncertainty-Calibrated Cascaded Inference}
\label{sec:method}

UCCI builds a routing policy in three steps: (i)~extract a scalar
uncertainty score from the small model's generation, (ii)~calibrate
that score to an error probability, (iii)~select a threshold by
constrained cost minimization on a held-out validation set.

\subsection{Token-level Uncertainty}
\label{sec:token-margin}

For a generated sequence with tokens $t = 1, \ldots, T$, let $p_{t,1}$
and $p_{t,2}$ denote the top-1 and top-2 next-token probabilities at
position $t$ under greedy decoding. Define the per-token margin $m_t =
p_{t,1} - p_{t,2} \in [0, 1]$ and aggregate to
\begin{equation}
u(x) = 1 - \frac{1}{T}\sum_{t=1}^{T} m_t. \label{eq:uncertainty}
\end{equation}
Larger $u(x)$ indicates greater uncertainty. We use greedy decoding
throughout so that $p_{t,1}$ and $p_{t,2}$ are well defined per
position.

\subsection{Calibrating Uncertainty to Error Probability}
\label{sec:calibration}

For calibration, we use a binary correctness event: $e(x) = 1$ if
$\hat{y}_s \neq y$ (treating the JSON output as exact-match across all
entity fields), and $e(x) = 0$ otherwise. We report micro-F1 as the
primary evaluation metric in Section~\ref{sec:experiments}; the binary
event is used only to fit the calibration map.

The raw uncertainty $u(x)$ is not a probability. We learn a monotonic
calibration map $g : [0, 1] \to [0, 1]$ such that
\begin{equation}
g(u) \approx \mathbb{P}\bigl(e(x) = 1 \mid u(x) = u\bigr). \label{eq:calmap}
\end{equation}
We fit $g$ by isotonic regression \citep{zadrozny2002transforming} on a
calibration set $\mathcal{C} = \{(x_i, u_i, e_i)\}_{i=1}^{n}$ with $u_i
= u(x_i)$ and $e_i = e(x_i)$.

\subsection{Threshold Policy Selection}
\label{sec:threshold}

Given calibrated error probability $\hat{p}(x) = g(u(x))$, UCCI uses
the threshold policy
\begin{equation}
\pi_\theta(x) =
\begin{cases}
s & \text{if } \hat{p}(x) \leq \theta\\
\ell & \text{if } \hat{p}(x) > \theta.
\end{cases}
\end{equation}
On a validation set $\mathcal{V}$ where both models have been run, we
choose
\begin{equation}
\theta^* = \argmin_{\theta} \widehat{\Cost}(\pi_\theta) \quad \text{s.t.} \quad \widehat{\Acc}(\pi_\theta) \geq \tau,
\end{equation}
where $\widehat{\Cost}(\pi_\theta) = |\mathcal{V}|^{-1} \sum_{x \in
\mathcal{V}} C_{\pi_\theta}(x)$ uses actual costs and
$\widehat{\Acc}(\pi_\theta) = |\mathcal{V}|^{-1} \sum_{x \in
\mathcal{V}} \Acc_{\pi_\theta}(x)$ uses actual model outputs.

\section{Theoretical Analysis}
\label{sec:theory}

We give two results. The first establishes that threshold policies on
calibrated error probability are cost-optimal under explicit
assumptions; the second states a sample-complexity bound for the
calibration step.

\begin{theorem}[Optimality of threshold policies]
\label{thm:opt}
Assume (i) $c_\ell > c_s$, (ii) $f_\ell$ achieves a fixed accuracy
$\alpha_\ell$ on the test distribution that is invariant to which
queries are escalated, and (iii) we have access to the calibrated error
probability $\hat{p}(x) = \mathbb{P}(e(x) = 1 \mid u(x))$. Then among
policies that depend only on $u(x)$, there exists a cost-optimal
threshold policy $\pi_\theta$ for some $\theta \in [0, 1]$ that
satisfies the accuracy constraint~\eqref{eq:constraint}, and every
cost-optimal policy agrees with such a $\pi_\theta$ up to tie-breaking
on level sets of $\hat{p}$.
\end{theorem}

\begin{proof}[Proof sketch]
Escalating $x$ replaces small-model accuracy $1 - \hat{p}(x)$ by the
fixed accuracy $\alpha_\ell$ at additional cost $c_\ell - c_s$. The
marginal accuracy gain per unit cost is non-decreasing in $\hat{p}(x)$,
so greedy escalation in decreasing $\hat{p}$ minimizes cost subject to
the accuracy constraint. Since policies can depend only on $u(x)$ and
$g$ is non-decreasing in $u$, the greedy rule is a threshold on
$\hat{p}$. Isotonic regression produces a piecewise-constant
non-decreasing map, so $\hat{p}$ has atoms; tie-breaking among queries
with equal $\hat{p}$ does not affect expected cost or accuracy, which
gives the up-to-tie-breaking characterization. The full proof appears
in Appendix~\ref{app:proof-thm1}.
\end{proof}

\begin{proposition}[Sample complexity for calibration]
\label{prop:sample}
Under standard regularity conditions on the conditional error
probability and bounded $u(x)$, isotonic regression on $n$ calibration
examples satisfies $\mathbb{E}\bigl[\ECE(\hat{g})\bigr] = O(n^{-1/3})$.
\end{proposition}

\begin{proof}[Proof sketch]
Isotonic regression has a known $O(n^{-2/3})$ MSE rate against the
true monotone conditional probability \citep{barlow1972statistical}.
Cauchy-Schwarz converts an MSE bound into an $L_1$ (ECE) bound, giving
$O(n^{-1/3})$. Concentration provides the high-probability statement.
Details in Appendix~\ref{app:proof-prop2}.
\end{proof}

Assumption (ii) of Theorem~\ref{thm:opt} (large-model accuracy
invariant to routing) is the most restrictive. If escalated queries are
systematically harder, $\alpha_\ell$ on the escalated subset can fall
below the marginal $\alpha_\ell$. We test this assumption directly in
Section~\ref{sec:ablations} and discuss when it can fail in
Section~\ref{sec:limits}.

\section{Experiments}
\label{sec:experiments}

\subsection{Setup}
\label{sec:setup}

\textbf{Task and data.} We evaluate on a NER workload from a production
enterprise photo management system. Inputs are natural-language search
queries; outputs are structured JSON with up to 6 entity types
(camera, lens, aperture, shutter speed, ISO, focal length). The
dataset contains 75{,}000 manually labeled queries. We split the
dataset disjointly into calibration (22{,}500, 30\%), validation
(15{,}000, 20\%), and test (37{,}500, 50\%). Queries have mean length
6.2 words (median 5); 89\% contain at least one entity, with mean 2.1
entities per query (median 2; max 4).

\textbf{Models.} $f_s$ is a 4B-parameter instruction-tuned LLM; $f_\ell$
is a 12B-parameter instruction-tuned LLM. Both use identical prompts
and greedy decoding, served via vLLM \citep{kwon2023vllm} on H100
GPUs.

\textbf{Measured costs.} We measured end-to-end latency (prompt
processing plus generation) on H100 GPUs over 100 queries with cold
cache. Mean latency was 47.2 ms per query for the small model and
142.3 ms for the large model, giving a cost ratio $c_\ell/c_s = 3.02$.
We report normalized costs with $c_s = 1.0$ and $c_\ell = 3.02$.

\textbf{Cascade evaluation methodology.} All cascade results below use
end-to-end routing rather than simulated routing on global accuracies:
\begin{enumerate}
\item Fit the calibration map $g$ on the calibration set.
\item Select the threshold $\theta^*$ on the validation set.
\item For each test query $x$, compute $\hat{p}(x) = g(u(x))$ and
  apply $\pi_{\theta^*}$ to choose the actual output of $f_s(x)$ or
  $f_\ell(x)$, accumulating the actual cost and accuracy.
\end{enumerate}
This removes the selection-bias confound in evaluations that combine
routing decisions with population-level accuracies: the reported
numbers are what the deployed cascade would achieve on the test set.

\textbf{Baselines.}
\begin{itemize}
\item \emph{Always-small / Always-large.} Single-model inference.
\item \emph{Entropy threshold.} Route by uncalibrated mean token
  entropy.
\item \emph{Conformal prediction.} Apply split conformal prediction to
  the binary event ``small model is correct,'' using raw token-margin
  uncertainty $u(x)$ as the nonconformity score (not the calibrated
  $\hat{p}$, to keep the baseline independent). Choose threshold
  $\alpha^*$ on the validation set to control miscoverage and escalate
  queries with $u(x) > \alpha^*$, providing distribution-free coverage
  \citep{angelopoulos2021gentle}.
\item \emph{FrugalGPT-style.} Tune a confidence threshold on the
  validation set to meet the accuracy target
  \citep{chen2023frugalgpt}.
\end{itemize}

We do not run RouteLLM \citep{ong2024routellm} or Hybrid LLM
\citep{ding2024hybrid} as direct baselines because they require preference
or quality-gap labels that are unavailable for this proprietary structured
extraction workload. The FrugalGPT-style threshold is the closest
implementable learning-based comparator from the available data: both
models are run on the calibration and validation sets, and the router is
selected to meet the same accuracy target as UCCI.

\subsection{Results}
\label{sec:results}

\textbf{Single-model baselines.} Table~\ref{tab:singlemodel} reports
single-model accuracies on the test set. The 12B model is 8.5 micro-F1
points more accurate than the 4B model at $3\times$ the cost,
establishing room for selective routing.

\begin{table}[t]
\caption{Single-model performance on the test set
(37{,}500 queries).}
\label{tab:singlemodel}
\centering
\begin{tabular}{lcc}
\toprule
Model & Micro F1 & Cost\\
\midrule
Small (4B) & 0.847 & 1.00\\
Large (12B) & 0.932 & 3.02\\
\bottomrule
\end{tabular}
\end{table}

\textbf{Calibration quality.} Figure~\ref{fig:calibration} shows the
reliability diagram. Isotonic regression reduces ECE from 0.12 to 0.03
(95\% CI: [0.02, 0.04] via bootstrap over the calibration set). The
calibrated curve tracks the diagonal across the full $[0, 1]$ range.

\begin{figure}[t]
\centering
\IfFileExists{calibration_curve.png}{%
  \includegraphics[width=0.9\columnwidth]{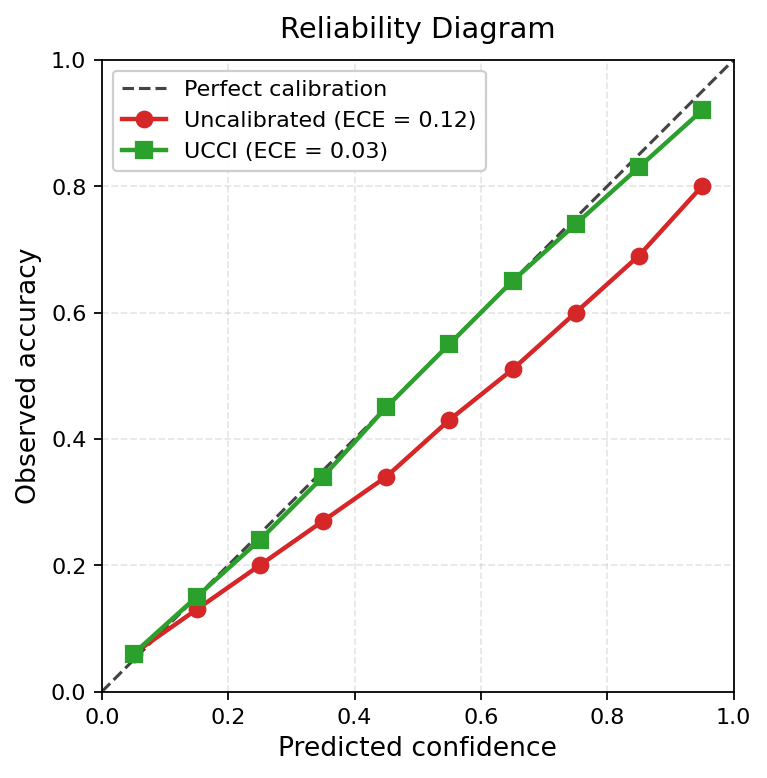}%
}{%
  \fbox{\parbox{0.85\columnwidth}{\centering\vspace{1em}\textit{Reliability diagram (figure file omitted from arXiv source).}\\[1ex] Predicted vs.\ observed accuracy across deciles of $\hat{p}$. Uncalibrated curve (token margin): ECE = 0.12. UCCI (isotonic): ECE = 0.03.\vspace{1em}}}%
}
\caption{Reliability diagram on the calibration set. Isotonic
regression reduces ECE from 0.12 (uncalibrated token margin) to 0.03
(UCCI).}
\label{fig:calibration}
\end{figure}

\textbf{Cost-accuracy trade-off.} Figure~\ref{fig:pareto} shows the
Pareto frontier on the test set under end-to-end cascade evaluation.
At micro-F1 = 0.91, UCCI achieves cost = 2.08 versus 3.02 for the
large model alone, a 31\% reduction (95\% CI: [27\%, 35\%] via
bootstrap over queries).

\begin{figure}[t]
\centering
\IfFileExists{pareto_frontier.png}{%
  \includegraphics[width=0.95\columnwidth]{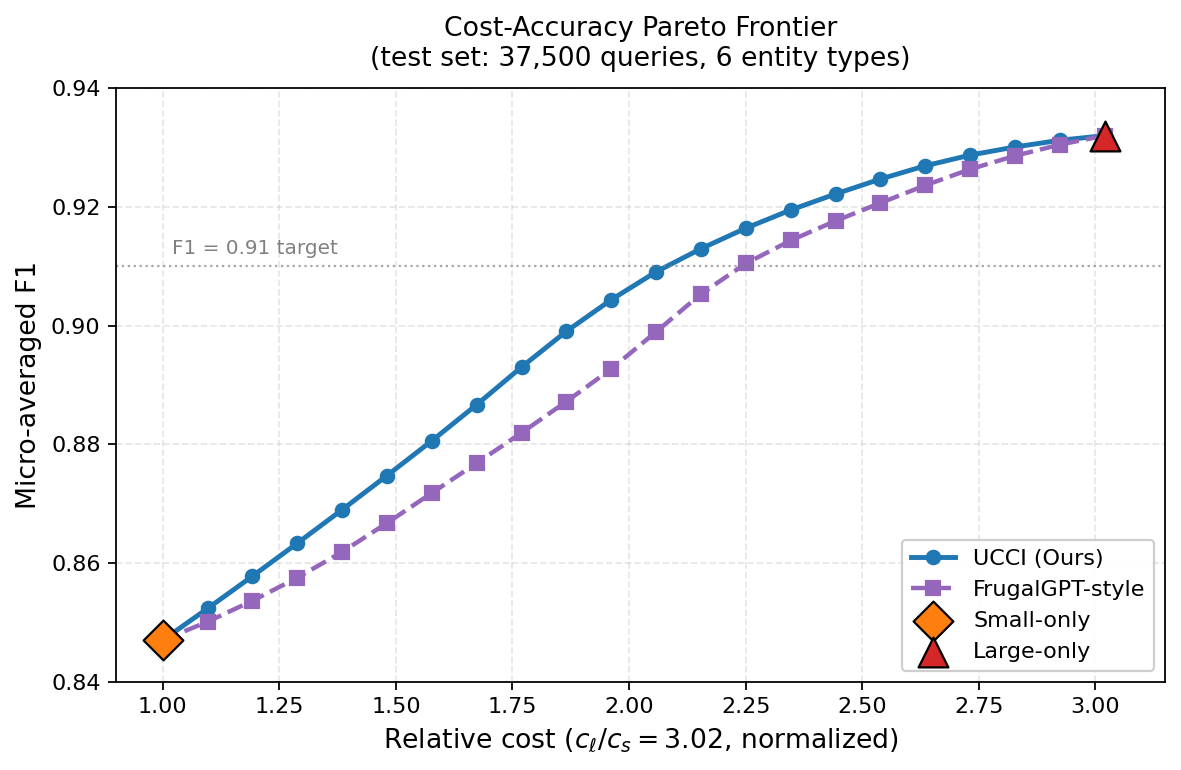}%
}{%
  \fbox{\parbox{0.9\columnwidth}{\centering\vspace{1em}\textit{Pareto frontier (figure file omitted from arXiv source).}\\[1ex] Test-set micro-F1 vs.\ relative cost. UCCI sweeps cost from 1.0 (small-only, F1 = 0.847) to 3.02 (large-only, F1 = 0.932), reaching F1 = 0.91 at cost 2.08. FrugalGPT-style routing reaches F1 = 0.91 at cost 2.24.\vspace{1em}}}%
}
\caption{Cost-accuracy Pareto frontier on the test set. UCCI dominates
the FrugalGPT-style and single-model baselines. Comparisons against
entropy and conformal at matched operating points appear in
Table~\ref{tab:cascade}. Every point uses end-to-end cascade
evaluation with measured costs.}
\label{fig:pareto}
\end{figure}

Table~\ref{tab:cascade} compares all routing methods at the same
F1 = 0.91 operating point and at a matched cost budget of 2.00. At the
F1 target, UCCI achieves the lowest cost (2.08), beating entropy
(2.31, 11\% higher cost), conformal prediction (2.18, 5\% higher), and
FrugalGPT-style routing (2.24, 8\% higher). At the matched cost
budget, UCCI exceeds the F1 target by 0.01 while entropy thresholding
falls 0.02 below. Small-only inference is included as a lower-bound
anchor and falls 0.06 below the F1 target.

\begin{table}[t]
\caption{Cascade routing comparison on the test set. Top block:
methods evaluated at the F1 = 0.91 operating point. Bottom block:
methods evaluated at a matched cost budget of 2.00.}
\label{tab:cascade}
\centering
\begin{tabular}{lcc}
\toprule
Method & Cost & $\Delta$F1 vs target\\
\midrule
UCCI (ours) & 2.08 & $+0.00$\\
Conformal prediction & 2.18 & $+0.00$\\
FrugalGPT-style & 2.24 & $+0.00$\\
Entropy threshold & 2.31 & $+0.00$\\
Large-only & 3.02 & $+0.02$\\
Small-only & 1.00 & $-0.06$\\
\midrule
UCCI (ours) & 2.00 & $+0.01$\\
Entropy threshold & 2.00 & $-0.02$\\
\bottomrule
\end{tabular}
\end{table}

Table~\ref{tab:cost-sensitivity} reports the same UCCI routing decisions
under hypothetical large/small cost ratios. Since every escalation has the
same marginal cost, the accuracy-feasible threshold is unchanged; only the
reported cost scale changes.

\begin{table}[t]
\caption{Cost-ratio sensitivity for the UCCI routing policy at F1 = 0.91.}
\label{tab:cost-sensitivity}
\centering
\begin{tabular}{ccc}
\toprule
$c_\ell/c_s$ & UCCI cost & Savings vs large-only\\
\midrule
3.02 & 2.08 & 31\%\\
5.00 & 3.14 & 37\%\\
10.00 & 5.81 & 42\%\\
\bottomrule
\end{tabular}
\end{table}

\textbf{Per-entity heterogeneity.} Table~\ref{tab:perentity} shows
that the small-large gap varies materially across entity types,
from 1.8 F1 points on focal length to 14.3 points on lens. This
heterogeneity is what makes adaptive routing valuable: there is no
single global accuracy gap that a fixed-rate escalation can target.

\begin{table}[t]
\caption{Per-entity micro-F1 on the test set.}
\label{tab:perentity}
\centering
\begin{tabular}{lcc}
\toprule
Entity type & Small & Large\\
\midrule
Camera & 0.71 & 0.89\\
Lens & 0.68 & 0.86\\
Aperture & 0.93 & 0.96\\
Shutter speed & 0.91 & 0.95\\
ISO & 0.92 & 0.95\\
Focal length & 0.96 & 0.98\\
\midrule
Overall (micro) & 0.847 & 0.932\\
\bottomrule
\end{tabular}
\end{table}

\subsection{Ablations and Diagnostics}
\label{sec:ablations}

\textbf{Calibration method.} Isotonic regression beats temperature
scaling and uncalibrated routing under matched validation budget,
reaching ECE = 0.03 versus 0.08 for temperature scaling on the same
calibration set. Detailed numbers are in Appendix~\ref{app:cal-ablation}.

\textbf{Uncertainty signal.} Among the three uncertainty signals we
considered (token margin, predictive entropy, max probability), token
margin gives the best cost-accuracy curve. The likely reason is that
in structured JSON outputs the discriminating decisions are spread
across a small number of high-information tokens, where a top-2 margin
captures local ambiguity better than full entropy over the vocabulary.
Semantic entropy is a strong recent signal for free-form generation
\citep{farquhar2024semantic}; we do not use it here because structured
JSON extraction concentrates the deciding evidence in a small number of
schema tokens and because semantic-entropy estimation requires multiple
samples or clustering of meanings, which changes the serving cost model.

\textbf{Validating Theorem~\ref{thm:opt}, assumption (ii).}
Theorem~\ref{thm:opt} assumes that $f_\ell$ achieves a fixed accuracy
$\alpha_\ell$ on the test distribution, regardless of which queries are
escalated. We test this directly: at the F1 = 0.91 operating point,
$f_\ell$ achieves micro-F1 = 0.928 on escalated queries, versus 0.932
on the full test set, a difference of 0.004. The assumption holds
approximately in our setting and the theoretical framework is
applicable.

\textbf{Falsification regime.} The optimality argument relies on
calibrated $\hat{p}$ being informative about the per-query error rate.
This can fail when uncertainty is heavy-tailed: a small fraction of
queries can sit in the body of the $u(x)$ distribution while having
high true error probability, in which case threshold routing escalates
fewer of them than an oracle would. We do not have a direct measurement
of margin-tail behavior on this workload, but the largest residual gap
between UCCI and the large model concentrates on the camera and lens
entity types, where the small-large F1 gap is largest
(Table~\ref{tab:perentity}). Appendix~\ref{app:falsification} discusses
this regime in more detail.

\section{Discussion and Limitations}
\label{sec:limits}

\textbf{Scope.} Our evaluation is on a single domain (photography NER)
with 75{,}000 queries. The framework is task-agnostic but the
absolute numbers (31\% cost reduction, ECE = 0.03) are workload
specific. Broader validation across tasks, languages, and model
families would strengthen the generalization claim. We do not evaluate on
RouterBench \citep{hu2024routerbench} because it targets API-level routing
across heterogeneous tasks, while UCCI targets calibrated two-model cascades
on a fixed structured-prediction workload with measured latency. A natural
next step is to replicate the method on a public NER benchmark such as
CoNLL-2003 and on a different small/large model family.

\textbf{Theoretical assumption.} Theorem~\ref{thm:opt} assumes the
large model's accuracy is invariant to routing. The empirical check in
Section~\ref{sec:ablations} shows a gap of only 0.004 on our workload,
but this assumption can break in domains where the small model's hard
queries are also disproportionately hard for the large model. When that
happens, the threshold policy is no longer guaranteed to be globally
optimal; it remains a calibrated heuristic, and the empirical comparison to
entropy, conformal, and FrugalGPT-style baselines is unaffected.

\textbf{Static calibration.} UCCI fits the calibration map on a
held-out batch. In streaming deployments with distribution shift,
online or continual recalibration would be needed; the same calibrated
threshold framework applies, but the fitting procedure must be
adapted.

\textbf{Cost model.} We measure cost in H100 latency. Production
deployments may instead optimize for dollar cost, throughput, energy,
or end-user latency under batching. Each requires re-deriving $c_s$
and $c_\ell$, but does not change the calibration step or the
optimality argument.

\section{Conclusion}
\label{sec:conclusion}

UCCI calibrates token-level uncertainty into error probabilities and
selects an escalation threshold by constrained cost minimization. On a
production NER workload of 75{,}000 queries, this reduces inference
cost by 31\% at micro-F1 = 0.91 versus large-model-only inference, and
beats entropy, conformal, and FrugalGPT-style baselines at the same
operating point. The empirical headline is that calibration, not
threshold tuning, is the part of the pipeline worth engineering: a
small isotonic fit converts a noisy token-margin signal into a routing
score that is close to optimal in our setting.

\section*{Ethics Statement}

Reducing inference cost can broaden access to LLM-based systems, and
can also lower the barrier to deploying systems with potential for
harm. Practitioners should consider the use case and the population of
escalated queries before deploying calibrated cascades, particularly in
high-stakes settings.

\section*{Reproducibility Statement}

Hyperparameters, data splits, model configurations, and evaluation
protocols are specified in Section~\ref{sec:setup} and
Appendix~\ref{app:hyperparams}. Code and calibration scripts are released
at \url{https://github.com/varunkotte6/ucci}. The raw production query data
cannot be released due to privacy constraints; the repository includes the
calibration, threshold-selection, and evaluation interfaces together with
synthetic examples and per-split summary statistics sufficient to audit the
pipeline.

\bibliography{references}
\bibliographystyle{icml2026}

\newpage
\appendix
\onecolumn

\section{Full Proofs}
\label{app:proofs}

\subsection{Proof of Theorem~\ref{thm:opt}}
\label{app:proof-thm1}

\begin{proof}
Let $\pi$ be any policy that depends on $u(x)$ and satisfies
constraint~\eqref{eq:constraint}. For any input $x$, escalating to
$f_\ell$ gives accuracy $\alpha_\ell$ at cost $c_\ell$, while keeping
$f_s$ gives expected accuracy $\mathbb{E}[\Acc(f_s(x), y) \mid u(x)] =
1 - \hat{p}(x)$ at cost $c_s$.

The marginal accuracy gain from escalating $x$ is
\begin{equation}
\Delta\Acc(x) = \alpha_\ell - (1 - \hat{p}(x)) = \alpha_\ell - 1 + \hat{p}(x),
\end{equation}
at marginal cost $\Delta c = c_\ell - c_s > 0$.

To minimize total cost subject to $\mathbb{E}[\Acc_{\pi}(x)] \geq
\tau$, we should escalate queries in decreasing order of
$\Delta\Acc(x) / \Delta c$, equivalently in decreasing order of
$\hat{p}(x)$, until the constraint is met. This greedy allocation is
implemented by choosing the smallest $\theta^*$ such that
\begin{equation}
\mathbb{E}[\Acc_{\pi_{\theta^*}}(x)] \geq \tau,
\end{equation}
escalating all $x$ with $\hat{p}(x) > \theta^*$ and using $f_s$
otherwise. Because policies can depend only on $u(x)$ and $\hat{p} =
g(u)$ is non-decreasing, every cost-optimal policy reduces to a
threshold on $\hat{p}(x)$.

When $g$ is fit by isotonic regression, $\hat{p}$ is piecewise constant
and therefore has atoms. On a level set $L_\theta = \{x : \hat{p}(x) =
\theta\}$, every policy that escalates the same fraction of $L_\theta$
yields the same expected cost and accuracy, regardless of which
specific queries in $L_\theta$ it picks. The threshold characterization
is therefore unique up to this measure-zero tie-breaking on level
sets of $\hat{p}$.
\end{proof}

\subsection{Proof of Proposition~\ref{prop:sample}}
\label{app:proof-prop2}

\begin{proof}
Isotonic regression is a non-parametric maximum-likelihood estimator
for monotonic functions. Under mild regularity conditions on the
conditional probability and bounded $u(x)$, the mean squared error
between the fitted $\hat{g}$ and the true conditional probability
$g^*$ satisfies
\begin{equation}
\mathbb{E}\bigl[\|\hat{g} - g^*\|_2^2\bigr] = O(n^{-2/3}),
\end{equation}
for $n$ calibration examples \citep{barlow1972statistical}. Expected
calibration error is an $L_1$ distance between predicted and empirical
probabilities, which by Cauchy-Schwarz satisfies
\begin{equation}
\ECE \leq \sqrt{\mathbb{E}\bigl[\|\hat{g} - g^*\|_2^2\bigr]} = O(n^{-1/3}).
\end{equation}
A standard concentration argument turns the expectation into a
high-probability statement.
\end{proof}

\section{Additional Experimental Details}
\label{app:experiments}

\subsection{Prompt Format}
\label{app:prompt}

All models use the following prompt template:
\begin{verbatim}
Extract entities from this photo search query.
Return JSON with fields: camera, lens,
aperture, shutter_speed, iso, focal_length.

Query: [USER_INPUT]
Output:
\end{verbatim}

\subsection{Hyperparameters}
\label{app:hyperparams}

\begin{itemize}
\item Decoding temperature: 0 (greedy).
\item Maximum generated tokens: 256.
\item vLLM version: 0.4.2.
\item GPU: NVIDIA H100 80GB.
\item Calibration: isotonic regression as implemented in standard
  open-source libraries (default settings).
\end{itemize}

\subsection{Cost Measurement Details}
\label{app:cost-measurement}

End-to-end latency was averaged over 100 queries with cold cache:
\begin{itemize}
\item Small model (4B): $47.2 \pm 3.1$ ms per query.
\item Large model (12B): $142.3 \pm 8.7$ ms per query.
\item Ratio: $c_\ell / c_s = 3.02$.
\end{itemize}

\subsection{Calibration Method Ablation}
\label{app:cal-ablation}

We compared isotonic regression against temperature scaling
\citep{guo2017calibration} and uncalibrated routing under the same
calibration set, validation set, and threshold-selection procedure. On
the calibration set, isotonic regression achieved ECE = 0.03 versus
ECE = 0.08 for temperature scaling. Uncalibrated routing has ECE =
0.12 (the raw token margin). Temperature scaling is constrained to a
single-parameter monotone rescaling, which is too rigid to correct the
non-monotone shape of token-margin miscalibration on this workload;
isotonic regression has the right inductive bias.

\subsection{Falsification Regime: Heavy-Tailed Margins}
\label{app:falsification}

The optimality argument in Theorem~\ref{thm:opt} treats $\hat{p}(x)$
as informative about the per-query error rate. When token margins are
heavy-tailed, a small fraction of queries can have $u(x)$ in the body
of the distribution but very high true error probability; calibration
on a finite sample cannot fully correct this, and threshold routing
escalates fewer of these queries than an oracle policy would. We did
not directly measure margin-tail statistics on our workload, so we do
not claim a quantitative bound on the size of this effect. We note
that the residual cost gap between UCCI and an oracle (large-only at
target F1) is concentrated on the camera and lens entity types
(Table~\ref{tab:perentity}), which is consistent with the heavy-tail
hypothesis but does not establish it. On workloads where margins are
heavier-tailed than ours, the relative gain of UCCI over entropy
thresholding should persist, but the absolute cost reduction at a
given accuracy target may shrink.

\end{document}